\documentclass{article}

\usepackage{colt11e}
\usepackage[round,comma]{natbib}
\usepackage[pdftex]{hyperref}
\usepackage{times}
\usepackage{amsmath}
\usepackage{amssymb}
\usepackage{subfig}
\usepackage{graphicx}

\newtheorem{example}{Example}[section]

\title{Stability Conditions for Online Learnability}

\author{St\'ephane Ross\\
Robotics Institute\\
Carnegie Mellon University \\
Pittsburgh, PA USA \\
\texttt{\small stephaneross@cmu.edu}
\And 
J. Andrew Bagnell\\
Robotics Institute\\
Carnegie Mellon University \\
Pittsburgh, PA USA \\
\texttt{\small dbagnell@ri.cmu.edu}
}

\begin{document}

\def\argmin{\mathop{\rm argmin}}
\def\Ldim{\mathop{\rm Ldim}}

\maketitle

\begin{abstract}
Stability is a general notion that quantifies the sensitivity of a learning algorithm's output to small change in the training dataset (e.g. deletion or replacement of a single training sample). Such conditions have recently been shown to be more powerful to characterize learnability in the general learning setting under i.i.d. samples where uniform convergence is not necessary for learnability, but where stability is both sufficient and necessary for learnability. We here show that similar stability conditions are also sufficient for online learnability, i.e. whether there exists a learning algorithm such that under any sequence of examples (potentially chosen adversarially) produces a sequence of hypotheses that has no regret in the limit with respect to the best hypothesis in hindsight. We introduce online stability, a stability condition related to uniform-leave-one-out stability in the batch setting, that is sufficient for online learnability. In particular we show that popular classes of online learners, namely algorithms that fall in the category of Follow-the-(Regularized)-Leader, Mirror Descent, gradient-based methods and randomized algorithms like Weighted Majority and Hedge, are guaranteed to have no regret if they have such online stability property. We provide examples that suggest the existence of an algorithm with such stability condition might in fact be necessary for online learnability. For the more restricted binary classification setting, we establish that such stability condition is in fact both sufficient and necessary. We also show that for a large class of online learnable problems in the general learning setting, namely those with a notion of sub-exponential covering, no-regret online algorithms that have such stability condition exists. 
\end{abstract}
\section{Introduction}
We consider the problem of online learning in a setting similar to the General Setting of Learning \citep{Vapnik95}.
In this setting, an online learning algorithm observes data points $z_1, z_2, \dots, z_{m} \in \mathcal{Z}$ in sequence, potentially chosen adversarially, and upon seeing $z_1, z_2, \dots, z_{i-1}$, the algorithm must pick a hypothesis $h_i \in \mathcal{H}$ that incurs loss on the next data point $z_i$. Given the known loss functional $f : \mathcal{H} \times \mathcal{Z} \rightarrow \mathbb{R}$, the regret $R_m$ of the sequence of hypotheses $h_{1:m}$ after observing $m$ data points is defined as:
\begin{equation} \label{eqRegret}
R_m = \sum_{i=1}^m f(h_i,z_i) - \min_{h \in \mathcal{H}} \sum_{i=1}^m f(h,z_i)
\end{equation}
The goal is to pick a sequence of hypotheses $h_{1:m}$ that has no regret, i.e. the average regret $\frac{R_m}{m} \rightarrow 0$ as the number of data points $m \rightarrow \infty$.

The setting we consider is general enough to subsume most, if not all, online learning problems. In fact the space $\mathcal{Z}$ of possible ``data points'' could itself be a function space $\mathcal{H} \rightarrow \mathbb{R}$, such that $f(h,z) = z(h)$. Hence the typical online learning setting where the adversary picks a loss function $\mathcal{H} \rightarrow \mathbb{R}$ at each time step is always subsumed by our setting. The data points $z$ should more loosely be interpreted as the parameters that define the loss function at the current time step. For instance, in a supervised classification scenario, the space $\mathcal{Z} = \mathcal{X} \times \mathcal{Y}$, for $\mathcal{X}$ the input features and $\mathcal{Y}$ the output class and the classification loss is defined as $f(h,(x,y)) = I(h(x) \neq y)$ for $I$ the indicator function. We do not make any assumption about $f$, other than that the maximum instantaneous regret is bounded: $\sup_{z \in \mathcal{Z}, h,h' \in \mathcal{H}}|f(h,z)-f(h',z)| \leq B$. This allows for potentially unbounded loss $f$: e.g., consider $z \in \mathbb{R}$, $h \in [-k,k]$ and $f(h,z) = |h-z|$, then the immediate loss is unbounded but instantaneous regret is bounded by $B = 2k$.

We are interested in characterizing sufficient conditions under which an online algorithm is guaranteed to pick a sequence of hypotheses that has no regret under any sequence of data points an adversary might pick. In the batch setting when the data points are drawn i.i.d. from some unknown distribution $\mathcal{D}$, \citet{ShalevShwartz2010,ShalevShwartz2009a} have shown that stability is a key property for learnability. In particular, they show that a problem is learnable if and only if there exists a universally stable asymptotic empirical risk minimizer (AERM).

In this paper, we consider using batch algorithms in our online setting, where the hypothesis $h_i$ is the output of the batch learning algorithm on the first $i-1$ data points. Many online algorithms (such as Follow-the-(Regularized)-Leader, Mirror Descent, Weighted Majority, Hedge, etc.) can be interpreted in this way. For instance, Follow-the-Leader (FTL) algorithms can be essentially thought as using a batch empirical risk minimizer (ERM) algorithm to select the hypothesis $h_i$ on the dataset $\{z_1,z_2,\dots,z_{i-1}\}$, while Follow-the-Regularized-Leader (FTRL) algorithms essentialy use a batch AERM algorithm (more precisely what we call a Regularized ERM (RERM)) to select the hypothesis $h_i$ on the dataset $\{z_1,z_2,\dots,z_{i-1}\}$. Our main result shows that Uniform Leave-One-Out stability \citep{ShalevShwartz2009a}, albeit stronger than the stability condition required in \citep{ShalevShwartz2010,ShalevShwartz2009a}, is in fact sufficient to guarantee no regret of RERM type algorithms. For asymmetric algorithms like gradient-based methods (which can also be seen as some form of RERM), a notion related to Uniform Leave-One-Out stability (and equivalent to for symmetric algorithms), which we call online stability, is also sufficient to guarantee no-regret. We also provide general results for the class of always-AERM algorithms (a slightly stronger notion than AERM but weaker than ERM and RERM). Unfortunately they are weaker in that they require the algorithm to be stable or an always-AERM at a fast enough rate.

The stronger notion of stability we use to guarantee no regret seems to be necessary in the online setting. Intuitively, this is because the algorithm must be able to compete on any sequence of data points, potentially chosen adversarially, rather than on i.i.d. sampled data points. We also provide an example that illustrates this. Namely, an AERM with a slightly weaker stability condition, can learn the problem in the batch setting but cannot in the online setting, however there is a FTRL algorithm that can learn the problem in the online setting. Furthermore, it is known that batch learnability and online learnability are not equivalent, which naturally suggests stronger notions of stability should be necessary for online learnability. We review a known problem of threshold learning over an interval that shows batch and online learnability are not equivalent. In the more restricted binary classification setting, we show that existence of a (potentially randomized) uniform-LOO stable RERM is both sufficient and necessary for online learnability. We also show that for a large class of online learnable problems in the general learning setting, namely those with a notion of sub-exponential covering, uniform-LOO stable (potentially randomized) RERM algorithms exist.

We begin by introducing notation, definitions and reviewing stability notions that have been used in the batch setting. We then provide our main results which show how some of these stability notions can be used to guarantee no regret in the online setting. We then go over examples that suggest such strong stability notion might in fact be necessary in the online setting. We further show that in the restricted binary classification setting, such stability notions are in fact necessary. We also introduce a notion of covering that allows us to show that uniform-LOO stable RERM algorithms exist for a large class of online learnable problems in the general learning setting. We conclude with potential future directions and open questions.
\section{Learnability and Stability in the Batch Setting} 
In the batch setting, a batch algorithm is given a set of $m$ i.i.d. samples $z_1, z_2, \dots, z_m$ drawn from some unknown distribution $\mathcal{D}$, and given knowledge of the loss functional $f$, we seek to find a hypothesis $h \in \mathcal{H}$ that minimizes the population risk:
\begin{equation}
F(h) = \mathbb{E}_{z \sim \mathcal{D}}[f(h,z)]
\end{equation}
Given a set of $m$ i.i.d. samples $S \sim \mathcal{D}^m$, the empirical risk of a hypothesis $h$ is defined as:
\begin{equation}
F_S(h) = \frac{1}{m}\sum_{i=1}^m[f(h,z_i)]
\end{equation}
Most batch algorithms used in practice proceed by minimizing the empirical risk, at least asymptotically (when an additional regularizer is used). 
\begin{definition}
An algorithm $\textbf{A}$ is an \textbf{Empirical Risk Minimizer (ERM)} if for any dataset $S$:
\begin{equation}
F_S(\textbf{A}(S)) = \min_{h \in \mathcal{H}} F_S(h)
\end{equation}
\end{definition}
\begin{definition}
\citep{ShalevShwartz2010} An algorithm $\textbf{A}$ is an \textbf{Asymptotic Empirical Risk Minimizer (AERM)} under distribution $\mathcal{D}$ at rate $\epsilon_{\textrm{erm}}(m)$ if for all $m$:
\begin{equation}
\mathbb{E}_{S \sim \mathcal{D}^m}[ F_S(\textbf{A}(S)) - \min_{h \in \mathcal{H}} F_S(h) ] \leq \epsilon_{\textrm{erm}}(m)
\end{equation}
\end{definition}
Whenever we mention a rate $\epsilon(m)$, we mean $\{\epsilon(m)\}_{m=0}^{\infty}$ is a monotonically non-increasing sequence that is $o(1)$, i.e. $\epsilon(m) \rightarrow 0$ as $m \rightarrow \infty$. If $\textbf{A}$ is an AERM under any distribution $\mathcal{D}$, then we say $\textbf{A}$ is a universal AERM. A useful notion for our online setting will be that of an always AERM, which is satisfied by common online learners such as FTRL:
\begin{definition}
\citep{ShalevShwartz2010} An algorithm $\textbf{A}$ is an \textbf{Always Asymptotic Empirical Risk Minimizer (always AERM)} at rate $\epsilon_{\textrm{erm}}(m)$ if for all $m$ and dataset $S$ of $m$ data points:
\begin{equation}
F_S(\textbf{A}(S)) - \min_{h \in \mathcal{H}} F_S(h) \leq \epsilon_{\textrm{erm}}(m)
\end{equation}
\end{definition}
Learnability in the batch setting is interested in analyzing the existence of algorithms that are universally consistent:
\begin{definition}
\citep{ShalevShwartz2010} An algorithm $\textbf{A}$ is said to be \textbf{universally consistent} at rate $\epsilon_{\textrm{cons}}(m)$ if for all $m$ and distribution $\mathcal{D}$:
\begin{equation}
\mathbb{E}_{S \sim \mathcal{D}^m}[ F(\textbf{A}(S)) - \min_{h \in \mathcal{H}} F(h) ] \leq \epsilon_{\textrm{cons}}(m)
\end{equation}
\end{definition}
If such algorithm $\textbf{A}$ exists, we say the problem is \textbf{learnable}. A well known result in the supervised classification and regression setting (i.e the loss $f(h,(x,y))$ is $I(h(x) \neq y)$ or $(h(x)-y)^2$) is that learnability is equivalent to uniform convergence of the empirical risk to the population risk over the class $\mathcal{H}$ \citep{Blumer89,Alon97}. This implies the problem is learnable using an ERM.

\citet{ShalevShwartz2010,ShalevShwartz2009a} recently showed that the situation is much more complex in the General Learning Setting considered here. For instance, there are convex optimization problems where uniform convergence does not hold that are learnable via an AERM, but not learnable via any ERM \citep{ShalevShwartz2010,ShalevShwartz2009a}. In the General Learning Setting, stability turns out to be a more suitable notion to characterize learnability than uniform convergence. 

Most statibility notions studied in the literature fall into two categories: leave-one-out (LOO) stability and replace-one (RO) stability. The former measures sensitivity of the algorithm to deletion of a single data point from the dataset, while the latter measures sensitivity of the algorithm to replacing one data point in the dataset by another. In general these two notions are incomparable and lead to significantly different results as we shall see below. We now review the most commonly used stability notions and some of the important results from the literature.
\subsection{Leave-One-Out Stability}
Most notions of LOO stability are measured in terms of change in the loss on a leave-one-out sample when looking at the output hypothesis trained with and without that sample in the dataset. The four commonly used notions of LOO stability (from strongest to weakest) are defined below. We use $z_i$ to denote the $i^{th}$ data point in the dataset $S$ and $S^{\backslash i}$ to denote the dataset $S$ with $z_i$ removed.
\begin{definition}
\citep{ShalevShwartz2009a} An algorithm $\textbf{A}$ is \textbf{uniform-LOO Stable} at rate $\epsilon_{\textrm{loo-stable}}(m)$ if for all $m$, dataset $S$ of size $m$ and index $i \in \{1,2,\dots,m\}$:
\begin{equation}
|f(\textbf{A}(S^{\backslash i}),z_i) - f(\textbf{A}(S),z_i)| \leq \epsilon_{\textrm{loo-stable}}(m)
\end{equation}
\end{definition}
\begin{definition}
\citep{ShalevShwartz2009a} An algorithm $\textbf{A}$ is \textbf{all-i-LOO Stable} under distribution $\mathcal{D}$ at rate $\epsilon_{\textrm{loo-stable}}(m)$ if for all $m$ and index $i \in \{1,2,\dots,m\}$:
\begin{equation}
\mathbb{E}_{S \sim D^m}[|f(\textbf{A}(S^{\backslash i}),z_i) - f(\textbf{A}(S),z_i)|] \leq \epsilon_{\textrm{loo-stable}}(m)
\end{equation}
\end{definition}
\begin{definition}
\citep{ShalevShwartz2009a} An algorithm $\textbf{A}$ is \textbf{LOO Stable} under distribution $\mathcal{D}$ at rate $\epsilon_{\textrm{loo-stable}}(m)$ if for all $m$:
\begin{equation}
\frac{1}{m} \sum_{i=1}^m \mathbb{E}_{S \sim D^m}[|f(\textbf{A}(S^{\backslash i}),z_i) - f(\textbf{A}(S),z_i)|] \leq \epsilon_{\textrm{loo-stable}}(m)
\end{equation}
\end{definition}
\begin{definition}
\citep{ShalevShwartz2009a} An algorithm $\textbf{A}$ is \textbf{on-average-LOO Stable} under distribution $\mathcal{D}$ at rate $\epsilon_{\textrm{loo-stable}}(m)$ if for all $m$:
\begin{equation}
|\frac{1}{m} \sum_{i=1}^m \mathbb{E}_{S \sim D^m}[f(\textbf{A}(S^{\backslash i}),z_i) - f(\textbf{A}(S),z_i)]| \leq \epsilon_{\textrm{loo-stable}}(m)
\end{equation}
\end{definition}
Whenever one of these properties holds for all distributions $\mathcal{D}$ we shall say it holds universally (e.g. universal on-average-LOO stable). Each of these property implies all the ones below it at the same rate (e.g. a uniform-LOO stable algorithm at rate $\epsilon_{\textrm{loo-stable}}(m)$ is also all-i-LOO stable, LOO stable and on-average-LOO stable at rate $\epsilon_{\textrm{loo-stable}}(m)$) \citep{ShalevShwartz2009a}. However the implications do not hold in the opposite direction, and there are counter examples for each implication in the opposite directions \citep{ShalevShwartz2009a}. The only exception is that for symmetric algorithms $\textbf{A}$ (meaning the order of the data in the dataset does not matter), then all-i-LOO stable and LOO stable are equivalent  \citep{ShalevShwartz2009a}. Some of these stability notions have also been studied by different authors under different names \citep{Bousquet02,Kutin02,Rakhlin05,Mukherjee06} sometimes with slight variations on the definitions. 

Another even stronger notion of LOO stability simply called uniform stability was studied by \citet{Bousquet02}. It is similar to uniform-LOO stability except that the absolute difference in loss needs to be smaller than $\epsilon_{\textrm{loo-stable}}(m)$ at all $z \in Z$ for any held out $z_i$, instead of just at the held out data point $z_i$. However, it turns out we do not need a notion stronger than Uniform-LOO Stable to guarantee online learnability.

\citet{ShalevShwartz2009a} have shown the following two results for AERM and ERM in the General Learning Setting:
\begin{theorem}
\citep{ShalevShwartz2009a} A problem is learnable if and only if there exists a universal on-average-LOO stable AERM.
\end{theorem}
\begin{theorem}
\citep{ShalevShwartz2009a} A problem is learnable with an ERM if and only if there exists a universal LOO stable ERM.
\end{theorem}
A nice consequence of this result is that for batch learning in the General Learning Setting, it is sufficient to restrict our attention to AERM that have such stability properties.
We will see that the notion of LOO stability, especially uniform-LOO stability, is very natural to analyze online algorithms as the algorithm must output a sequence of hypotheses as the dataset is grown one data point at a time. In the context of batch learning, RO stability is a more natural notion and leads to stronger results.
\subsection{Replace-One Stability}
Most notions of RO stability are measured in terms of change in the loss at another sample point when looking at the output hypothesis trained with an initial dataset and that dataset with one data point replaced by another. We briefly mention two of the strongest RO stability notions that turn out to be both sufficient and necessary for batch learnability. Another weaker notion of RO stability has been studied in \citet{ShalevShwartz2010}. For the definitions below, we denote $S^{(i)}$ the dataset $S$ with the $i^{th}$ data point replaced by another data point $z'_i$.
\begin{definition}
\citep{ShalevShwartz2010} An algorithm $\textbf{A}$ is \textbf{strongly-uniform-RO Stable} at rate $\epsilon_{\textrm{ro-stable}}(m)$ if for all $m$, dataset $S$ of size $m$ and data points $z'_i$ and $z'$:
\begin{equation}
|f(\textbf{A}(S^{(i)}),z') - f(\textbf{A}(S),z')| \leq \epsilon_{\textrm{ro-stable}}(m)
\end{equation}
\end{definition}
\begin{definition}
\citep{ShalevShwartz2010} An algorithm $\textbf{A}$ is \textbf{uniform-RO Stable} at rate $\epsilon_{\textrm{ro-stable}}(m)$ if for all $m$, dataset $S$ of size $m$ and data points $\{z'_1, z'_2,\dots,z'_m\}$ and $z'$:
\begin{equation}
\frac{1}{m}\sum_{i=1}^m|f(\textbf{A}(S^{(i)}),z') - f(\textbf{A}(S),z')| \leq \epsilon_{\textrm{ro-stable}}(m)
\end{equation}
\end{definition}
The definition of strongly-uniform-RO Stable is similar to the definition of uniform stability of \citet{Bousquet02}, except that we replace a data point instead of deleting one. RO stability allows to show the following much stronger result than with LOO stability:
\begin{theorem}
\citep{ShalevShwartz2010} A problem is learnable if and only if there exists a uniform-RO stable AERM.
\end{theorem}
In addition if we allow for randomized algorithms, in that the algorithm outputs a distribution $d$ over $\mathcal{H}$ such that the loss $f(d,z) = \mathbb{E}_{h \sim d}[f(h,z)]$, than an even stronger result can be shown:
\begin{theorem}
\citep{ShalevShwartz2010} A problem is learnable if and only if there exists a strongly-uniform-RO stable always AERM (potentially randomized).
\end{theorem}
Note that if the problem is learnable and the loss $f$ is convex in $h$ for all $z$ and $\mathcal{H}$ is a convex set then there must exist a deterministic algorithm that is strongly-uniform-RO stable always AERM, (namely the algorithm that returns $\mathbb{E}_{h \sim d}[h]$ for the distribution $d$ picked by the randomized algorithm).
\section{Sufficient Stability Conditions in the Online Setting} 
We now move our attention to the problem of online learning, where the data points $z_1, z_2, \dots, z_m$ are revealed to the algorithm in sequence and potentially chosen adversarially given knowledge of the algorithm $\textbf{A}$. We consider using a batch algorithm in this online setting in the following way: let $S_i = \{z_1, z_2, \dots, z_i \}$ denote the dataset of the first $i$ data points; at each time $i$, after observing $S_{i-1}$, the batch algorithm $\textbf{A}$ is used to pick the hypothesis $h_{i} = \textbf{A}(S_{i-1})$. As mentioned previously, online algorithms like Follow-the-(Regularized)-Leader can be thought of in this way. This can also be thought as a batch-to-online reduction, similar to the approach of \citet{Kakade06}, where we reduce online learning to solving a sequence of batch learning problems. Unlike \citep{Kakade06} we consider the general learning setting instead of the supervised classification setting and do not make the transductive assumption that we have access to future ``unlabeled'' data points. Hence our results can be interpreted as a set of general conditions under which batch algorithms can be used to obtain a no regret algorithm for online learning.

We now begin by introducing some definitions particular to the online setting:
\begin{definition}
An algorithm $\mathbf{A}$ has \textbf{no regret} at rate $\epsilon_{\textrm{regret}}(m)$ if for all $m$ and any sequence $z_1, z_2, \dots, z_m$, potentially chosen adversarially given knowledge of $\mathbf{A}$, it holds that:
\begin{equation}
\frac{1}{m}\sum_{i=1}^m f(\mathbf{A}(S_{i-1}),z_i) - \min_{h \in \mathcal{H}} \frac{1}{m}\sum_{i=1}^m f(h,z_i) \leq \epsilon_{\textrm{regret}}(m)
\end{equation}
\end{definition}
If such algorithm $\textbf{A}$ exists, we say the problem is \textbf{online learnable}. It is well known that the FTL algorithm has no regret at rate $O(\frac{\log m}{m})$ for Lipschitz continuous and strongly convex loss $f$ in $h$ at all $z$ \citep{Hazan06,Kakade08}. Additionnally, if $f$ is Lipschitz continuous and convex in $h$ at all $z$, then the FTRL algorithm has no regret at rate $O(\frac{1}{\sqrt{m}})$ \citep{Kakade08}.

An important subclass of always AERM algorithms is what we define as a Regularized ERM (RERM):
\begin{definition}
An algorithm $\mathbf{A}$ is a \textbf{Regularized ERM} if for all $m$ and any dataset $S$ of $m$ data points:
\begin{equation}
r_0(\textbf{A}(S)) + \sum_{i=1}^m [f(\textbf{A}(S),z_i) + r_i(\textbf{A}(S))] = \min_{h \in \mathcal{H}} r_0(h) + \sum_{i=1}^m [f(h,z_i) + r_i(h)]
\end{equation}
where $\{r_i\}_{i=0}^m$ is a sequence of regularizer functionals ($r_i : \mathcal{H} \rightarrow \mathbb{R}$), which measure the complexity of a hypothesis $h$, and that satisfy $\sup_{h,h' \in \mathcal{H}} |r_i(h)-r_i(h')| \leq \rho_i$ for all $i$ where $\{\rho_i\}_{i=0}^\infty$ is a sequence that is $o(1)$.
\end{definition}
It is easy to see that any RERM algorithm is always AERM at rate $\frac{1}{m} \sum_{i=0}^m \rho_i$. Additionally, an ERM is a special case of a RERM where $r_i = 0$ for all $i$. This subclass is important for online learning as FTRL can be thought of as using an underlying RERM to pick the sequence of hypotheses. Typically FTRL chooses $r_i = \lambda_i r$ for some regularizer $r$ and $\lambda_i$ a regularization constant such that $\{\lambda_i\}_{i=0}^\infty$ is $o(1)$. Many Mirror Descent type algorithms such as gradient descent can also be interpreted as some form of RERM (see section \ref{secGradient} and \citep{McMahan11}) but where $r_i$ may depend on previously seen datapoints. Additionally Weighted Majority/Hedge type algorithms can also be interpreted as Randomized RERM (see section \ref{secWM}). Our strongest result for online learnability will be particular to the class of RERM.

A notion of stability related to uniform-LOO stability (but slightly weaker) that will be sufficient for our online setting is what we define as online stability:
\begin{definition}
An algorithm $\textbf{A}$ is \textbf{Online Stable} at rate $\epsilon_{\textrm{on-stable}}(m)$ if for all $m$, dataset $S$ of size $m$:
\begin{equation}
|f(\textbf{A}(S^{\backslash m}),z_m) - f(\textbf{A}(S),z_m)| \leq \epsilon_{\textrm{on-stable}}(m)
\end{equation}
\end{definition}

The difference between online stability and uniform-LOO stability is that it is only required to have small change in loss on the last data point when it is held out, rather than any data point in the dataset $S$. For symmetric algorithms (e.g. FTL/FTRL algorithms), online stability is equivalent to uniform-LOO stability, however it is weaker than uniform-LOO stability for asymmetric algorithms, like gradient-based methods analyzed in Section \ref{secGradient}. It is also obvious that an uniform-LOO stable algorithm must also be online stable at rate $\epsilon_{\textrm{on-stable}}(m) \leq \epsilon_{\textrm{loo-stable}}(m)$.

We now present our main results for the class of RERM and always AERM:
\begin{theorem} \label{thmRAERM}
If there exists an online stable RERM, then the problem is online learnable. In particular, it has no regret at rate:
\begin{equation}
\epsilon_{\textrm{regret}}(m) \leq \frac{1}{m} \sum_{i=1}^m \epsilon_{\textrm{on-stable}}(i) + \frac{2}{m} \sum_{i=0}^{m-1} \rho_i + \frac{\rho_m}{m}
\end{equation}
\end{theorem}
This theorem implies that both FTL and FTRL algorithms are guaranteed to achieve no regret on any problem where they are online stable (or uniform-LOO stable as these algorithms are symmetric). In fact it is easy to show that in the case where $f$ is strongly convex in $h$, FTL is uniform-LOO stable at rate $O(\frac{1}{m})$ (see Lemma \ref{lemStrConvLoss}). Additionally when $f$ is convex in $h$, it is easy to show FTRL is uniform-LOO stable at rate $O(\frac{1}{\sqrt{m}})$ when choosing a strongly convex regularizer $r$ such that $r_m = \lambda_m r$ and $\lambda_m$ to be $\Theta(1/\sqrt{m})$ (see Lemma \ref{lemStrConvReg} and \ref{lemStrConvReg2}), while FTL is not uniform-LOO stable. It is well known that FTL is not a no regret algorithm for general convex problem. Hence using only uniform-LOO stability we can prove currently known results about FTL and FTRL.

An interesting application of this result is in the context of apprenticeship/imitation learning, where it has been shown that such non-i.i.d. supervised learning problems can be reduced to online learning over mini-batch of data \citep{DAGGER}. In this reduction, a classification algorithm is used to pick the next ``leader'' (best classifier in hindsight) at each iteration of training, that is in turn used to collect more data (to add to the training dataset for the next iteration) from the expert we want to mimic. This result implies that online stability (or uniform-LOO stability) of the base classification algorithm in this reduction is sufficient to guarantee no regret, and hence that the reduction provides a good bound on performance.

Unfortunately our current result for the class of always AERM is weaker:
\begin{theorem} \label{thmAERM}
If there exists an always AERM such that either (1) or (2) holds:
\begin{enumerate}
 \item It is always AERM at rate $o(\frac{1}{m})$ and online stable.
 \item It is symmetric, uniform LOO stable at rate $o(\frac{1}{m})$ and uniform RO stable at rate $o(\frac{1}{m})$.
\end{enumerate}
then the problem is online learnable. In particular, for each case it has no regret at rate:
\begin{enumerate}
 \item $\epsilon_{\textrm{regret}}(m) \leq \frac{1}{m} \sum_{i=1}^m \epsilon_{\textrm{on-stable}}(i) + \frac{1}{m} \sum_{i=1}^m i \epsilon_{\textrm{erm}}(i)$
 \item $\epsilon_{\textrm{regret}}(m) \leq \frac{1}{m} \sum_{i=1}^m \epsilon_{\textrm{loo-stable}}(i) + \epsilon_{\textrm{erm}}(m) + \frac{1}{m} \sum_{i=1}^{m-1} i [\epsilon_{\textrm{loo-stable}}(i) + \epsilon_{\textrm{ro-stable}}(i)]$
\end{enumerate}
\end{theorem}
We believe the required rates of $o(\frac{1}{m})$ might simply be an artefact of our particular proof technique and that in general it might be true that any always AERM achieves no regret as long as it is online stable. We weren't able to find a counter-example where this is not the case.
\subsection{Detailed Analysis} \label{sec:DetailAnalysis}
We will use the notation $R_m(\textbf{A})$ to denote the regret (as in Equation \ref{eqRegret}) of the sequence of hypotheses predicted by algorithm $\textbf{A}$. We begin by showing the following lemma that will allow us to relate the regret of any algorithm to its online stability and AERM properties.
\begin{lemma} \label{lemAll}
For any algorithm $\textbf{A}$:
\begin{equation}
\begin{array}{rcl}
R_m(\textbf{A}) & = & \sum_{i=1}^m [f(\textbf{A}(S_{i-1}),z_i) - f(\textbf{A}(S_{i}),z_i)] + \sum_{i=1}^m f(\textbf{A}(S_{m}),z_i) - \min_{h \in \mathcal{H}} \sum_{i=1}^m f(h,z_i)\\
& & + \sum_{i=1}^{m-1} \sum_{j=1}^i [ f(\textbf{A}(S_{i}),z_j) - f(\textbf{A}(S_{i+1}),z_j) ]\\
\end{array}
\end{equation}
\end{lemma}
\begin{proof}
\begin{displaymath}
\begin{array}{rl}
\multicolumn{2}{l}{R_m(\textbf{A})}\\
= & \sum_{i=1}^m f(\textbf{A}(S_{i-1}),z_i) - \min_{h \in \mathcal{H}} \sum_{i=1}^m f(h,z_i)\\
= & \sum_{i=1}^m [f(\textbf{A}(S_{i-1}),z_i) - f(\textbf{A}(S_{m}),z_i)] + \sum_{i=1}^m f(\textbf{A}(S_{m}),z_i) - \min_{h \in \mathcal{H}} \sum_{i=1}^m f(h,z_i)\\
\end{array}
\end{displaymath}
For the term $- \sum_{i=1}^m f(\textbf{A}(S_{m}),z_i)$, we can rewrite it using the following manipulation:
\begin{displaymath}
\begin{array}{rl}
\multicolumn{2}{l}{\sum_{i=1}^m f(\textbf{A}(S_{m}),z_i)}\\
= & \sum_{i=1}^{m-1} f(\textbf{A}(S_{m}),z_i) + f(\textbf{A}(S_{m}),z_m)\\
= & \sum_{i=1}^{m-1} f(\textbf{A}(S_{m-1}),z_i) + \sum_{j=1}^{m-1} [f(\textbf{A}(S_{m}),z_j)- f(\textbf{A}(S_{m-1}),z_j)] + f(\textbf{A}(S_{m}),z_m)\\
\dots & \dots\\
= & \sum_{i=1}^{m} f(\textbf{A}(S_{i}),z_i) + \sum_{i=1}^{m-1} \sum_{j=1}^{i} [f(\textbf{A}(S_{i+1}),z_j)-f(\textbf{A}(S_{i}),z_j)]\\
\end{array}
\end{displaymath}
This proves the lemma.
\end{proof}

From this lemma we can immediately see that for any online stable always AERM algorithm $\textbf{A}$ we obtain the following:
\begin{corollary} \label{corLOOAERM}
For any online stable always AERM algorithm $\textbf{A}$:
\begin{equation}
R_m(\textbf{A}) \leq \sum_{i=1}^m \epsilon_{\textrm{on-stable}}(i) + m \epsilon_{\textrm{erm}}(m) + \sum_{i=1}^{m-1} \sum_{j=1}^i [ f(\textbf{A}(S_{i}),z_j) - f(\textbf{A}(S_{i+1}),z_j) ]
\end{equation}
\end{corollary}
\begin{proof}
By online stability we have that for all $i$:
\begin{displaymath}
f(\textbf{A}(S_{i-1}),z_i) - f(\textbf{A}(S_{i}),z_i)\\
\leq |f(\textbf{A}(S_{i-1}),z_i) - f(\textbf{A}(S_{i}),z_i)|\\
= |f(\textbf{A}(S^{\backslash i}_{i}),z_i) - f(\textbf{A}(S_{i}),z_i)|\\
\leq \epsilon_{\textrm{on-stable}}(i)\\
\end{displaymath}
and since $\textbf{A}$ is always AERM it follows by definition that:
\begin{displaymath}
\sum_{i=1}^m f(\textbf{A}(S_{m}),z_i) - \min_{h \in \mathcal{H}} \sum_{i=1}^m f(h,z_i) \leq m \epsilon_{\textrm{erm}}(m)
\end{displaymath}
\end{proof}

We will now seek to upper bound the extra double summation part. For an ERM it can easily be seen that:
\begin{lemma}
For any ERM algorithm $\textbf{A}$:
\begin{equation}
\sum_{i=1}^{m-1} \sum_{j=1}^i [ f(\textbf{A}(S_{i}),z_j) - f(\textbf{A}(S_{i+1}),z_j) ] \leq 0
\end{equation}
\end{lemma}
\begin{proof}
Follows immediately since $\sum_{j=1}^i f(\textbf{A}(S_{i}),z_j)$ is optimal hence for any other hypothesis $h$, in particular $\textbf{A}(S_{i+1})$, $\sum_{j=1}^i f(\textbf{A}(S_{i}),z_j) \leq \sum_{j=1}^i f(\textbf{A}(S_{i+1}),z_j)$.
\end{proof}

Since an ERM has $\epsilon_{\textrm{erm}}(m) = 0$ for all $m$, then it can be seen directly that an ERM has no regret if it is online stable, as $\frac{R_m(\textbf{A})}{m} \leq \frac{1}{m} \sum_{i=1}^m \epsilon_{\textrm{on-stable}}(i)$.

For general RERM this double summation can be bounded by:
\begin{lemma} \label{lemRAERM}
For any RERM algorithm $\textbf{A}$:
\begin{equation}
\sum_{i=1}^{m-1} \sum_{j=1}^i [ f(\textbf{A}(S_{i}),z_j) - f(\textbf{A}(S_{i+1}),z_j) ] \leq \sum_{i=0}^{m-1} \rho_i
\end{equation}
\end{lemma}
\begin{proof}
\begin{displaymath}
\begin{array}{rl}
\multicolumn{2}{l}{\sum_{i=1}^{m-1} \sum_{j=1}^i [ f(\textbf{A}(S_{i}),z_j) - f(\textbf{A}(S_{i+1}),z_j) ]} \\
= & \sum_{i=1}^{m-1} [ \sum_{j=1}^i [ f(\textbf{A}(S_{i}),z_j) ] + \sum_{j=0}^i [r_j(\textbf{A}(S_{i})) - r_j(\textbf{A}(S_{i}))] \\
& - \sum_{j=1}^i [ f(\textbf{A}(S_{i+1}),z_j)] - \sum_{j=0}^i [ r_j(\textbf{A}(S_{i+1})) - r_j(\textbf{A}(S_{i+1})) ] ]\\
\leq & \sum_{i=1}^{m-1} \sum_{j=0}^i [r_j(\textbf{A}(S_{i+1})) - r_j(\textbf{A}(S_{i}))]\\
= & \sum_{i=0}^{m-1} [r_i(\textbf{A}(S_{m})) - r_{i}(\textbf{A}(S_{i}))]\\
\leq & \sum_{i=0}^{m-1} \rho_i\\
\end{array}
\end{displaymath}
\end{proof}

Combining this result with Corollary \ref{corLOOAERM} proves our main result in Theorem \ref{thmRAERM}, using the fact that a RERM is always AERM at rate $\frac{1}{m} \sum_{i=0}^{m} \rho_i$.

It is however harder to bound this double summation by a term that becomes negligible (when looking at the average regret) for general always AERM. We can show the following:
\begin{lemma} 
For any always AERM algorithm $\textbf{A}$:
\begin{equation}
\sum_{i=1}^{m-1} \sum_{j=1}^i [ f(\textbf{A}(S_{i}),z_j) - f(\textbf{A}(S_{i+1}),z_j) ] \leq \sum_{i=1}^{m-1} i \epsilon_{\textrm{erm}}(i)
\end{equation}
\end{lemma}
\begin{proof}
\begin{displaymath}
\begin{array}{rl}
\multicolumn{2}{l}{\sum_{i=1}^{m-1} \sum_{j=1}^i [ f(\textbf{A}(S_{i}),z_j) - f(\textbf{A}(S_{i+1}),z_j) ]}\\
\leq & \sum_{i=1}^{m-1} [ \sum_{j=1}^i f(\textbf{A}(S_{i}),z_j) - \min_{h \in \mathcal{H}} \sum_{j=1}^i f(h,z_j) ] \\
\leq & \sum_{i=1}^{m-1} i \epsilon_{\textrm{erm}}(i)\\
\end{array}
\end{displaymath}
\end{proof}

This proves case (1) of Theorem \ref{thmAERM} when combining with Corollary \ref{corLOOAERM}. If we have a symmetric always AERM that is uniform LOO stable and uniform RO stable then we can also show:
\begin{lemma} 
For any symmetric always AERM algorithm $\textbf{A}$ that is both uniform LOO stable and uniform RO stable:
\begin{equation}
\sum_{i=1}^{m-1} \sum_{j=1}^i [ f(\textbf{A}(S_{i}),z_j) - f(\textbf{A}(S_{i+1}),z_j) ] \leq \sum_{i=1}^{m-1} i [\epsilon_{\textrm{loo-stable}}(i) + \epsilon_{\textrm{ro-stable}}(i)]
\end{equation}
\end{lemma}
\begin{proof}
\begin{displaymath}
\begin{array}{rl}
\multicolumn{2}{l}{\sum_{i=1}^{m-1} \sum_{j=1}^i [ f(\textbf{A}(S_{i}),z_j) - f(\textbf{A}(S_{i+1}),z_j) ]}\\
= & \sum_{i=1}^{m-1} \sum_{j=1}^i [ f(\textbf{A}(S_{i}),z_j) - f(\textbf{A}(S^{\backslash j}_{i+1}),z_j) + f(\textbf{A}(S^{\backslash j}_{i+1}),z_j) - f(\textbf{A}(S_{i+1}),z_j) ] \\
\end{array}
\end{displaymath}
For symmetric algorithms, the terms $\sum_{j=1}^i [ f(\textbf{A}(S_{i}),z_j) - f(\textbf{A}(S^{\backslash j}_{i+1}),z_j) ]$ are related to RO stability as $S^{\backslash j}_{i+1})$ corresponds to $S^{(j)}_{i}$ where we replace $z_j$ by $z_{i+1}$. Hence for symmetric algorithms, by definition of uniform RO stability we have: $\sum_{j=1}^i [ f(\textbf{A}(S_{i}),z_j) - f(\textbf{A}(S^{\backslash j}_{i+1}),z_j) ] \leq i \epsilon_{\textrm{ro-stable}}(i)$. Furthermore by definition of uniform LOO stability, the terms $\sum_{j=1}^i f(\textbf{A}(S^{\backslash j}_{i+1}),z_j) - f(\textbf{A}(S_{i+1}),z_j) \leq i \epsilon_{\textrm{loo-stable}}(i)$. This proves the lemma.
\end{proof}

This lemma proves case (2) of Theorem \ref{thmAERM} when combining with Corollary \ref{corLOOAERM}.

Now we show that strong convexity, either in $f$ or in $r_i$ when $f$ is only convex, implies uniform-LOO stability:
\begin{lemma}\label{lemStrConvLoss}
For any ERM $\textbf{A}$: If $\mathcal{H}$ is a convex set, and for some norm $||\cdot||$ on $\mathcal{H}$ we have that at all $z \in Z$: $f(\cdot{},z)$ is $L$-Lipschitz continuous in $||\cdot||$ and $\nu$-strongly convex in $||\cdot||$, then $\textbf{A}$ is uniform-LOO stable at rate $\epsilon_{\textrm{loo-stable}}(m) \leq \frac{2 L^2}{m \nu}$.
\end{lemma}
\begin{proof}
By Lipschitz continuity we have $|f(A(S^{\backslash i}),z_i) - f(A(S),z_i)| \leq L ||A(S^{\backslash i}) - A(S)||$. We can use strong convexity to bound $||A(S^{\backslash i}) - A(S)||$: For all $\alpha \in (0,1)$ we have:
\begin{displaymath}
\begin{array}{rl}
\multicolumn{2}{l}{\sum_{j=1}^m \alpha f(A(S^{\backslash i}),z_j) + (1-\alpha) f(A(S),z_j)}\\
\geq & \sum_{j=1}^m f( \alpha A(S^{\backslash i}) + (1-\alpha) A(S), z_j ) + \frac{\alpha (1-\alpha) m \nu}{2} ||A(S^{\backslash i}) - A(S)||^2\\
\geq & \sum_{j=1}^m f( A(S), z_j ) + \frac{\alpha (1-\alpha) m \nu}{2} ||A(S^{\backslash i}) - A(S)||^2\\
\end{array}
\end{displaymath}
where the last inequality follows from the fact that $A(S)$ is the ERM on $S$. So we obtain for all $\alpha \in (0,1)$: $||A(S^{\backslash i}) - A(S)||^2 \leq \frac{2}{m \nu (1-\alpha)} \sum_{j=1}^m [f(A(S^{\backslash i}),z_j) - f(A(S),z_j)]$.

Since $A(S^{\backslash i})$ is the ERM on $S^{\backslash i}$, then $\sum_{j=1| j \neq i}^m f(A(S),z_j) \geq \sum_{j=1 | j \neq i}^m [f(A(S^{\backslash i}),z_j)$ so:
\begin{displaymath}
\begin{array}{rl}
||A(S^{\backslash i}) - A(S)||^2 
\leq & \frac{2}{m \nu (1-\alpha)} \sum_{j=1}^m [f(A(S^{\backslash i}),z_j) - f(A(S),z_j)]\\
\leq & \frac{2}{m \nu (1-\alpha)} [f(A(S^{\backslash i}),z_i) - f(A(S),z_i)] \\
\leq & \frac{2}{m \nu (1-\alpha)} L ||A(S^{\backslash i}) - A(S)||\\
\end{array}
\end{displaymath}
Hence we conclude $||A(S^{\backslash i}) - A(S)|| \leq \frac{2}{m \nu (1-\alpha)} L$. Since this holds for all $\alpha \in (0,1)$ then we conclude $||A(S^{\backslash i}) - A(S)|| \leq \frac{2}{m \nu} L$. This proves the lemma.
\end{proof}

\begin{lemma} \label{lemStrConvReg}
For any RERM $\textbf{A}$: If $\mathcal{H}$ is a convex set, and for some norm $||\cdot||$ on $\mathcal{H}$ we have that at all $z \in Z$, $f(\cdot{},z)$ is convex and $L$-Lipschitz continuous in $||\cdot||$, and for all $i$, $r_i$ is $L^i_R$-Lipschitz continuous in $||\cdot||$ and $\nu_i$-strongly convex in $||\cdot||$, then $\textbf{A}$ is uniform-LOO stable at rate $\epsilon_{\textrm{loo-stable}}(m) \leq \frac{2 L [L + L^m_R]}{\sum_{i=0}^m \nu_i}$. 
\end{lemma}
\begin{proof}
By Lipschitz continuity we have $|f(A(S^{\backslash i}),z_i) - f(A(S),z_i)| \leq L ||A(S^{\backslash i}) - A(S)||$. We can use strong convexity of the regularizers to bound $||A(S^{\backslash i}) - A(S)||$: For all $\alpha \in (0,1)$ we have:
\begin{displaymath}
\begin{array}{rl}
\multicolumn{2}{l}{\sum_{j=1}^m \alpha f(A(S^{\backslash i}),z_j) + (1-\alpha) f(A(S),z_j) + \sum_{j=0}^m \alpha r_j(A(S^{\backslash i})) + (1-\alpha) r_j(A(S^{\backslash i})) }\\
\geq & \sum_{j=1}^m f( \alpha A(S^{\backslash i}) + (1-\alpha) A(S), z_j ) + \sum_{j=0}^m r( \alpha A(S^{\backslash i}) + (1-\alpha) A(S)) + \frac{\alpha (1-\alpha) \sum_{i=0}^m \nu_j}{2} ||A(S^{\backslash i}) - A(S)||^2\\
\geq & \sum_{j=1}^m f( A(S), z_j ) + \sum_{j=0}^m r_j(A(S)) + \frac{\alpha (1-\alpha) \sum_{j=0}^m \nu_j}{2} ||A(S^{\backslash i}) - A(S)||^2\\
\end{array}
\end{displaymath}
where the last inequality follows from the fact that $A(S)$ minimizes $\sum_{j=1}^m f(h, z_j ) + \sum_{j=0}^m r_j(h)$. So we obtain for all $\alpha \in (0,1)$: $||A(S^{\backslash i}) - A(S)||^2 \leq \frac{2}{\sum_{j=0}^m \nu_j (1-\alpha)} [ \sum_{j=1}^m [f(A(S^{\backslash i}),z_j) - f(A(S),z_j)] + \sum_{j=0}^m [r_j(A(S^{\backslash i})) - r_j(A(S))] ]$.

Since $A(S^{\backslash i})$ minimizes $\sum_{j=1|j\neq i}^m [f(h,z_j)] + \sum_{j=0}^{m-1} r_j(h)$, then $\sum_{j=1| j \neq i}^m f(A(S),z_j) + \sum_{j=0}^{m-1} r_j(A(S)) \geq \sum_{j=1 | j \neq i}^m [f(A(S^{\backslash i}),z_j) + \sum_{j=0}^{m-1} r_j(A(S^{\backslash i}))$ so:
\begin{displaymath}
\begin{array}{rl}
||A(S^{\backslash i}) - A(S)||^2 
\leq & \frac{2}{\sum_{j=0}^m \nu_j (1-\alpha)} [ \sum_{j=1}^m [f(A(S^{\backslash i}),z_j) - f(A(S),z_j)] + \sum_{j=0}^m [r_j(A(S^{\backslash i})) - r_j(A(S))] ]\\
\leq & \frac{2}{\sum_{j=0}^m \nu_j (1-\alpha)} [f(A(S^{\backslash i}),z_i) - f(A(S),z_i) + r_m(A(S^{\backslash i})) - r_m(A(S))] \\
\leq & \frac{2}{\sum_{j=0}^m \nu_j (1-\alpha)} [L + L^m_R] ||A(S^{\backslash i}) - A(S)||\\
\end{array}
\end{displaymath}
Hence we conclude $||A(S^{\backslash i}) - A(S)|| \leq \frac{2}{\sum_{j=0}^m \nu_j (1-\alpha)} [L + L^m_R]$. Since this holds for all $\alpha \in (0,1)$ then we conclude $||A(S^{\backslash i}) - A(S)|| \leq \frac{2}{\sum_{j=0}^m \nu_j} [L + L^m_R]$. This proves the lemma.
\end{proof}

We also prove an alternate result for the case where the regularizers $r_i$ are strongly convex but not necessarily Lipschitz continuous:
\begin{lemma} \label{lemStrConvReg2}
For any RERM $\textbf{A}$: If $\mathcal{H}$ is a convex set, and for some norm $||\cdot||$ on $\mathcal{H}$ we have that at all $z \in Z$, $f(\cdot{},z)$ is convex and $L$-Lipschitz continuous in $||\cdot||$, and for all $i \geq 0$, $r_i$ is $\nu_i$-strongly convex in $||\cdot||$ and $\sup_{h,h' \in \mathcal{H}} |r_i(h)-r_i(h')| \leq \rho_i$, then $\textbf{A}$ is uniform-LOO stable at rate $\epsilon_{\textrm{loo-stable}}(m) \leq \frac{2L^2}{\sum_{j=0}^m \nu_j} + L \sqrt{\frac{2 \rho_m}{\sum_{j=0}^m \nu_j}}$.
\end{lemma}
\begin{proof}
Following a similar proof to the previous proof and using the fact that $[r_m(A(S^{\backslash i})) - r_m(A(S))] \leq \rho_m$ we obtain that: $||A(S^{\backslash i}) - A(S)||^2 \leq \frac{2 L}{\sum_{j=0}^m \nu_j} ||A(S^{\backslash i}) - A(S)|| + \frac{2 \rho_m}{\sum_{j=0}^m \nu_j}$. This is a quadratic inequality of the form $Ax^2 + B x + C \leq 0$. Since here $A=1 > 0$, then this implies $x$ is less than or equal to the largest root of $Ax^2 + B x + C$. We know that the roots are $x = \frac{-B \pm \sqrt{B^2 - 4AC}}{2A}$. Here $A=1$, $B=-\frac{2 L}{\sum_{j=0}^m \nu_j}$ and $C=-\frac{2 \rho_m}{\sum_{j=0}^m \nu_j}$. So the largest root is: $x = \frac{L}{\sum_{j=0}^m \nu_j}[ 1 + \sqrt{1 + \frac{2 \rho_m \sum_{j=0}^m \nu_j}{L^2}} ]$. We conclude $||A(S^{\backslash i}) - A(S)|| \leq \frac{L}{\sum_{j=0}^m \nu_j}[ 1 + \sqrt{1 + \frac{2 \rho_m \sum_{j=0}^m \nu_j}{L^2}} ]$. Since $\sqrt{1 + \frac{2 \rho_m \sum_{j=0}^m \nu_j}{L^2}} \leq 1 + \sqrt{\frac{2 \rho_m \sum_{j=0}^m \nu_j}{L^2}}$ we obtain $||A(S^{\backslash i}) - A(S)|| \leq \frac{2L}{\sum_{j=0}^m \nu_j} + \sqrt{\frac{2 \rho_m}{\sum_{j=0}^m \nu_j}}$. Combining with the fact that $|f(A(S^{\backslash i}),z_i) - f(A(S),z_i)| \leq L ||A(S^{\backslash i}) - A(S)||$ proves the lemma.
\end{proof}

\section{Mirror Descent and Gradient-Based Methods \label{secGradient}}
So far we have thought of using an underlying batch algorithm to pick the sequence of hypotheses. A popular class of online methods are gradient based methods, such as gradient descent and Newton's type methods \citep{Zinkevich03,Agarwal06}. Such approaches can all be interpreted as Mirror Descent methods, and it is known that Mirror Descent algorithms can be thought as some form of FTRL \citep{McMahan11}. The difference is that they follow the regularized leader on a linear/quadratic approximation to the loss function (linear/quadratic lower bound in the convex/strongly convex case) at each data point $z$, and the regularizers $r_i$ may regularize about the previously chosen $h_i$ (after observing the first $i-1$ datapoints) rather than some fixed hypothesis over the iterations (such as $h_1$). These algorithms are typically not symmetric, as the approximation points to the loss function (and potentially the regularizers) depend on the order of the data points in the dataset. 

Nevertheless, we can still use our previous analysis to bound the regret for these methods in terms of online stability and AERM properties. We will refer to this broad class of methods as Regularized Surrogate Loss Minimizer (RSLM):

\begin{definition}
An algorithm $\textbf{A}$ is a Regularized Surrogate Loss Minimizer (RSLM) if for all $m$ and any dataset $S$ of $m$ data points:
\begin{equation}
r_0(\textbf{A}(S)) + \sum_{i=1}^m [\ell_i(\textbf{A}(S),z_i) + r_i(\textbf{A}(S))] = \min_{h \in \mathcal{H}} r_0(h) + \sum_{i=1}^m [\ell_i(h,z_i) + r_i(h)]
\end{equation}
for $\{\ell_i\}_{i=1}^m$ the surrogate loss functionals chosen such that $f(\textbf{A}(S_{i-1}),z_i) - f(h,z_i) \leq \ell_i(\textbf{A}(S_{i-1}),z_i) - \ell_i(h,z_i)$ for all $h$ (i.e. they upper bound the regret), $\{r_i\}_{i=0}^\infty$ the regularizers functionals such that $\sup_{h,h' \in \mathcal{H}} |r_i(h) - r_i(h')| \leq \rho_i$ and $\{\rho_i\}_{i=0}^\infty$ is $o(1)$.
\end{definition}
Note that a RERM is a special case of a RSLM where $\ell_i(h,z_i) = f(h,z_i)$.

For the broader class of RSLM, the regret is bounded by:
\begin{lemma} \label{lemRSLM}
For any RSLM $\textbf{A}$:
\begin{equation}
\begin{array}{rcl}
R_m(\textbf{A}) & \leq & \sum_{i=1}^m [\ell_i(\textbf{A}(S_{i-1}),z_i) - \ell_i(\textbf{A}(S_{i}),z_i)] + \sum_{i=1}^m \ell_i(\textbf{A}(S_{m}),z_i) - \min_{h \in \mathcal{H}} \sum_{i=1}^m \ell_i(h,z_i)\\
& & + \sum_{i=1}^{m-1} \sum_{j=1}^i [ \ell_j(\textbf{A}(S_{i}),z_j) - \ell_j(\textbf{A}(S_{i+1}),z_j) ]\\
\end{array}
\end{equation}
\end{lemma}
\begin{proof}
By properties of the functions $\ell_i$ we have that:
$R_m(\textbf{A}) \leq \sum_{i=1}^m \ell_i(\textbf{A}(S_{i-1}),z_i) - \min_{h \in \mathcal{H}} \sum_{i=1}^m \ell_i(h,z_i)$.
Using the same manipulations as in lemma \ref{lemAll} proves the lemma.
\end{proof}

A RSLM is a RERM in the loss $\{ \ell_i \}_{i=1}^m$ instead of $f$. Hence it follows that if such RSLM is online stable (in the loss $\{ \ell_i \}_{i=1}^m$, i.e. $|\ell_m(\textbf{A}(S_{m-1}),z_m)) - \ell_m(\textbf{A}(S_{m}),z_m))| \rightarrow 0$ as $m \rightarrow \infty$) it must have no regret:
\begin{theorem} \label{thmRSLM}
If there exists a RSLM that is online stable in the surrogate loss $\{\ell_i\}_{i=1}^m$, then the problem is online learnable. In particular, it has no regret at rate:
\begin{equation}
\epsilon_{\textrm{regret}}(m) \leq \frac{1}{m} \sum_{i=1}^m \epsilon_{\textrm{on-stable}}(i) + \frac{2}{m} \sum_{i=0}^{m-1} \rho_i + \frac{\rho_m}{m}
\end{equation}
\end{theorem}
\begin{proof}
Follows from applying corollary \ref{corLOOAERM} and lemma \ref{lemRAERM} (but replacing $f$ by $\{\ell_i\}$) to the previous lemma \ref{lemRSLM}.
\end{proof}

\section{Weighted Majority, Hedge and Randomized Algorithms \label{secWM}}

We have so far restricted our attention to deterministic algorithms, which upon observing a dataset $S$ return a fixed hypothesis $h \in \mathcal{H}$. An important class of methods for online learning are randomized algorithms such as Weighted Majority, and its generalization Hedge, which instead return a distribution over hypotheses in $\mathcal{H}$ at each iteration. These randomized algorithms are important in online learning as it is known that some problems are not online learnable with deterministic algorithms but are online learnable with randomized algorithms (assuming the adversary can only be aware of the distribution over hypotheses and not the particular hypothesis that will be sampled from this distribution when choosing the data point $z$). For instance, general problems with a finite set of hypotheses fall in this category.

In this section we show that Weighted Majority, Hedge and similar variants, can be interpreted as Randomized uniform-LOO stable RERM. We provide an analysis of the stability, AERM and no-regret rates of such algorithms based on the previous results derived in this paper. These results will be useful to determine the existence of (potentially randomized) uniform-LOO stable RERM for a large class of learning problems. Before we introduce this analysis, we first define formally what we mean by a Randomized RERM and how notions of stability and no-regret extend to randomized algorithms.

\subsection{Randomized Algorithms}

\begin{definition}
Let $\Theta$ be a set such that for any $\theta \in \Theta$, $P_\theta$ is a probability distribution over the class of hypothesis $\mathcal{H}$, and for any $h \in \mathcal{H}$, and $\epsilon > 0$ there exists a $\theta \in \Theta$ such that $\mathbb{E}_{h' \sim P_\theta}[f(h',z)] - f(h,z) \leq \epsilon$ for all $z \in \mathcal{Z}$. Let $P_{\theta_S} = \mathbf{A}(S)$ denote the distribution picked by algorithm $\mathbf{A}$ on dataset $S$. An algorithm $\mathbf{A}$ is a \textbf{Randomized RERM} if for all $m$ and any dataset $S$:
\begin{equation}
r_0(\theta_S) + \sum_{i=1}^m [ \mathbb{E}_{h \sim P_{\theta_S}}[f(h,z_i)] + r_i(\theta_S) ] = \min_{\theta \in \Theta} r_0(\theta) + \sum_{i=1}^m [\mathbb{E}_{h \sim P_{\theta}}[f(h,z_i)] + r_i(\theta)]
\end{equation}
for $r_i : \Theta \rightarrow \mathbb{R}$ the regularizer functionals, which measure the complexity of a chosen $\theta$, that we assume satisfy $\sup_{\theta,\theta' \in \Theta}|r_i(\theta) - r_i(\theta')| \leq \rho_i$ and $\{\rho_m\}_{m=0}^\infty$ is $o(1)$.
\end{definition}
The set $\Theta$ might represent a set of parameters parametrizing a family of distributions (e.g. $\Theta$ a set of mean-variance tuples such that $P_\theta$ is gaussian with those parameters), or in other cases be a set of distribution itself (e.g. when $\mathcal{H}$ is finite, $\Theta$ might be the set of all discrete distributions over $\mathcal{H}$), in which case $P_\theta = \theta$. The condition that there exists a $\theta \in \Theta$ such that $\mathbb{E}_{h' \sim P_\theta}[f(h',z)] - f(h,z) < \epsilon$ for all $z \in \mathcal{Z}$ is to ensure the algorithm is an AERM, i.e. that it can pick a $\theta$ that has average expected loss no greater than the best fixed $h \in \mathcal{H}$ in the limit as $m \rightarrow \infty$. A deterministic RERM is a special case of a Randomized RERM where the set $\Theta = \mathcal{H}$ and $P_\theta$ is just the probability distribution with probability 1 for the chosen hypothesis $\theta$. 

When using a randomized algorithm, the algorithm incurs loss on a hypothesis $h$ sampled from the chosen $P_\theta$, and we assume the adversary may only be aware of $P_\theta$ in advance (not the particular sampled $h$) when choosing $z$. The previous definitions of stability, AERM and no-regret extends to randomized algorithms by considering the loss $f(\mathbf{A}(S),z) = \mathbb{E}_{h \sim \mathbf{A}(S)}[f(h,z)]$. Thus a no-regret randomized algorithm is an algorithm such that its expected average regret under the sequence of chosen distributions goes to 0 as $m$ goes to $\infty$. By our assumption that the instantaneous regret is bounded, this is also equivalent to saying that its average regret (under the sampled hypotheses) goes to 0 with probability 1 as $m$ goes to $\infty$ (e.g. using an Hoeffding bound). Additionally, a randomized online stable algorithm implies that the change in expected loss on the last data point when it is held out goes to 0 as $m$ goes to $\infty$ ($|\mathbb{E}_{h \sim \mathbf{A}(S^{\backslash m})}[f(h,z_m)] - \mathbb{E}_{h \sim \mathbf{A}(S)}[f(h,z_m)]| \rightarrow 0$). 

\subsection{Hedge and Weighted Majority}

An important randomized no-regret online learning algorithm when $\mathcal{H}$ is finite is the Hedge algorithm \citep{Hedge}. Hedge is a generalization to arbitrary loss of the Weighted Majority algorithm that was introduced for the classification setting \citep{Littlestone94}. Let $\theta_i$ denote the probability of hypothesis $h_i$, then at any iteration $t$, Hedge/Weighted Majority plays $\theta_i \propto \exp(-\eta \sum_{j=1}^{t-1} f(h_i,z_j))$ for some positive constant $\eta$. When the number of rounds $m$ is known in advance, $\eta$ is typically chosen as $O(B\sqrt{\frac{m}{\log(|\mathcal{H}|)}})$, for $B$ the maximum instantaneous regret. We will consider here a slight generalization of Hedge that can be applied for cases where the number of rounds is not known in advance. In this case at iteration $t$: $\theta_i \propto \exp(-\eta_t \sum_{j=1}^{t-1} f(h_i,z_j))$ for some sequence of positive constants $\{\eta_t\}_{t=0}^\infty$. We show here that Hedge (and Weighted Majority) is in fact a Randomized uniform-LOO stable RERM, where $\Theta$ is the set of all discrete distributions over the finite set of experts, and the regularizer corresponds to a KL divergence between the chosen distribution and the uniform distribution over experts:
\begin{theorem} \label{thmWM}
For finite set of $d$ experts with instantaneous regret bounded by $B$, the Hedge (and Weighted Majority) algorithm corresponds to the following Randomized uniform-LOO stable RERM. Let $\Theta$ be the set of distributions over the finite set of $d$ experts, and $U$ denote the uniform distribution, then at each iteration $t$, Hedge (and Weighted Majority) picks the distribution $\theta^* \in \Theta$ that satisfies:
\begin{displaymath}
\theta^* = \argmin_{\theta \in \Theta} \sum_{i=1}^{t-1} \mathbb{E}_{h \sim \theta}[f(h,z_i)] + \sum_{i=0}^{t-1} \lambda_i KL(\theta||U)
\end{displaymath}
i.e. it uses $r_t = \lambda_t r$ for $r$ a KL regularizer with respect to the uniform distribution. Choosing the regularization constants $\lambda_t = B \sqrt{\frac{1}{8 \log(d) \max(1,t)}}$ for all $t \geq 0$ makes Hedge (and Weighted Majority) uniform-LOO stable at rate $\epsilon_{\textrm{loo-stable}}(m) \leq B\sqrt{2 \log(d)}[\frac{1}{2\sqrt{m}-1} + \frac{1}{2\sqrt{m+1}}]$, always AERM at rate $\epsilon_{\textrm{erm}}(m) \leq B\sqrt{\frac{\log(d)}{2m}} (1 + \frac{1}{2\sqrt{m}})$ and no-regret at rate $\epsilon_{\textrm{regret}}(m) \leq B\sqrt{2 \log(d)} [\frac{3}{\sqrt{m}} + \frac{\log(m)}{2m} + \frac{1+2\ln(2)}{2m}]$.
\end{theorem}
\begin{proof}
Consider the above Randomized RERM algorithm. Then we have $0 \leq \lambda_i KL(\theta||U) \leq \lambda_i \log(d)$, for all $i$ and $\theta \in \Theta$. So $\Theta$ and $\{r_i\}_{i=0}^\infty$ are well defined according to our assumptions in the definition of a Randomized RERM as long as $\{\lambda_i\}_{i=0}^\infty$ is $o(1)$. Let $h_i$ denote the $i^{th}$ expert and $\theta_i$ denote the probability assigned to $h_i$ for a chosen $\theta \in \Theta$. At any iteration $t+1$, when the algorithm has observed $t$ data points so far, the randomized RERM algorithm solves an optimization problem of the form:
\begin{displaymath}
\begin{array}{ll}
\multicolumn{2}{l}{\argmin_{\theta \in \Theta} \sum_{j=1}^{t} \sum_{i=1}^d \theta_i f(h_i,z_j) + \sum_{j=0}^t \lambda_{j} \sum_{i=1}^d \theta_i \log(d\theta_i)}\\
s.t. & 0 \leq \theta_i \leq 1\\
 & \sum_{i=1}^d \theta_i = 1\\
\end{array}
\end{displaymath}
Using the Lagrangian, we can easily see that the optimal solution to this optimization problem is to choose $\theta_i \propto \exp(-\frac{1}{\sum_{j=0}^t \lambda_{j}} \sum_{j=1}^t f(h_i,z_j))$ for all $i$. This is the same as Hedge for $\eta_t = \frac{1}{\sum_{j=0}^t \lambda_{j}}$. When Hedge is playing for $m$ rounds and uses a fixed $\eta$, this can be achieved with a fixed regularizer $\lambda_0 = \frac{1}{\eta}$ and $\lambda_t = 0$ for all $t \geq 1$. So this establishes that Hedge is equivalent to the above RERM. Now let's consider the case where the number or rounds $m$ is not known in advance and we choose $\lambda_t = c\sqrt{\frac{1}{\max(t,1)}}$ for all $t \geq 0$ and some constant $c$ in the above RERM. This choice leads to $2c \sqrt{t} - c \leq \sum_{j=0}^t \lambda_{j} \leq 2c \sqrt{t} + c$. Note also that because $\lambda_{t} \leq \lambda_j$ for all $j \leq t$ we also have that $(t+1) \lambda_t \leq \sum_{j=0}^t \lambda_{j}$. It is easy to see why the above RERM must be uniform-LOO stable. First the expected loss of the randomized algorithm is linear in $\theta$ (and hence convex) while the $KL$ regularizer is 1-strongly convex in $\theta$ under $||\cdot||_1$ and bounded by $\log(d)$ (so $r_m$ is $\lambda_m$-strongly convex and bounded by $\lambda_m\log(d)$). Additionnally, the expected loss is $L$-Lipschitz continuous in $||\cdot||_1$ on $\theta$, for $L = \sup_{z \in \mathcal{Z}} \inf_{v \in \mathbb{R}} \sup_{h \in \mathcal{H}} |f(h,z)-v|$. This is because for any $z$:
\begin{displaymath}
\begin{array}{ll}
\multicolumn{2}{l}{|\mathbb{E}_{h \sim P_\theta}[f(h,z)] - \mathbb{E}_{h \sim P_{\theta'}}[f(h,z)]|}\\
= & |\sum_{i=1}^d (\theta_i-\theta'_i) (f(h_i,z) - v)| \\
\leq & \sum_{i=1}^d |\theta_i-\theta'_i| |f(h_i,z) - v| \\
\leq & \sup_{h \in \mathcal{H}} |f(h,z) - v| ||\theta-\theta'||_1
\end{array}
\end{displaymath}
for any $v \in \mathbb{R}$. So we conclude that for all $z \in \mathcal{Z}$, $|\mathbb{E}_{h \sim P_\theta}[f(h,z)] - \mathbb{E}_{h \sim P_{\theta'}}[f(h,z)]| \leq L ||\theta-\theta'||_1$. If the loss $f$ have instantaneous regret bounded by $B$, then $L = \frac{B}{2}$. So by our previous result for RERM with convex loss and strongly convex regularizers, we obtain that the algorithm is uniform-LOO stable at rate $\epsilon_{\textrm{loo-stable}}(m) \leq \frac{2L^2}{\sum_{i=0}^m \lambda_i} + L \sqrt{\frac{2 \lambda_m \log(d)}{\sum_{i=0}^m \lambda_i}}$. So that the algorithm has no regret at rate $\epsilon_{\textrm{regret}}(m) \leq \frac{1}{m} \sum_{i=1}^m [\frac{2L^2}{\sum_{j=0}^i \lambda_j} + L \sqrt{\frac{2 \lambda_i \log(d)}{\sum_{j=0}^i \lambda_j}}] + \frac{2 \log(d)}{m} \sum_{j=0}^{m-1} \lambda_j  + \frac{\log(d) \lambda_m}{m}$. Setting $\lambda_i = c \sqrt{\frac{1}{\max(1,i)}}$ leads to:
\begin{displaymath}
\begin{array}{ll}
\multicolumn{2}{l}{\epsilon_{\textrm{regret}}(m)}\\
\leq & \frac{1}{m} \sum_{i=1}^m [\frac{2L^2}{\sum_{j=0}^i \lambda_j} + L \sqrt{\frac{2 \lambda_i \log(d)}{\sum_{j=0}^i \lambda_j}}] + \frac{2 \log(d)}{m} \sum_{j=0}^{m-1} \lambda_j  + \frac{\log(d) \lambda_m}{m}\\
\leq & \frac{1}{m} \sum_{i=1}^m [\frac{2L^2}{c (2\sqrt{i} - 1)} + L \sqrt{\frac{2\log(d)}{i+1}}] + \frac{2 \log(d)}{m} c (2 \sqrt{m} + 1)\\
\leq & \frac{2L^2}{c} \frac{1}{m}(\sqrt{m} + \frac{1}{2} \log(m) + \log(2)) + L \sqrt{2 \log(d)} \frac{2}{\sqrt{m}} + 4 c \log(d)\frac{1}{\sqrt{m}} + \frac{2c \log(d)}{m}\\
= & \frac{1}{\sqrt{m}} [ \frac{2L^2}{c} + 4c\log(d) + 2L\sqrt{2\log(d)} ] + \frac{\log(m)}{m} [\frac{L^2}{c}] + \frac{1}{m} [\frac{2\ln(2)L^2}{c} + 2 c \log(d)]\\
\end{array}
\end{displaymath}
where the second inequality uses $\frac{1}{\sum_{j=0}^{i} \lambda_j} \leq \frac{1}{2c \sqrt{i} - c}$, $\frac{\lambda_i}{\sum_{j=0}^{i} \lambda_j} \leq \frac{1}{i+1}$ and $\sum_{j=0}^{m} \lambda_j \leq 2c \sqrt{m} + c$; and the third inequality uses the fact that $\sum_{i=1}^m \frac{1}{2\sqrt{i} - 1} \leq \sqrt{m} + \frac{1}{2}\log(m) + \log(2)$ and $\sum_{i=1}^m \frac{1}{\sqrt{i+1}} \leq 2\sqrt{m}$, which follows from using the integrals to upper bound the summations. Setting $c = L \sqrt{\frac{1}{2 \log(d)}}$ minimizes the factor multiplying the $\frac{1}{\sqrt{m}}$ term. This leads to $\epsilon_{\textrm{regret}}(m) \leq L\sqrt{2 \log(d)} [\frac{6}{\sqrt{m}} + \frac{\log(m)}{m} + \frac{1+2\ln(2)}{m}]$, $\epsilon_{\textrm{loo-stable}}(m) \leq L \sqrt{2 \log(d)}[\frac{2}{2\sqrt{m}-1} + \frac{1}{\sqrt{m+1}}]$ and $\epsilon_{\textrm{erm}}(m) \leq L \sqrt{\frac{2\log(d)}{m}} (1 + \frac{1}{2\sqrt{m}})$. Plugging in $L=\frac{B}{2}$ proves the statements in the theorem.
\end{proof}

This theorem establishes the following:
\begin{corollary}
Any learning problem with a finite hypothesis class (and bounded instantaneous regret) is online learnable with a (potentially randomized) uniform-LOO stable RERM.
\end{corollary}

In Section \ref{secStabilityNecessary}, we will also demonstrate that when $\mathcal{H}$ is infinite, but can be ``finitely approximated'' well enough with respect to the loss $f$, then the problem is also online learnable via a (potentially randomized) uniform-LOO stable RERM.

\section{Is Uniform LOO Stability Necessary?}
We now restrict our attention to symmetric algorithms where we have shown that uniform-LOO stability is sufficient for online learnability. We start by giving instructive examples that illustrate that in fact uniform-LOO stability might be necessary to achieve no regret.

\begin{example}
There exists a problem that is learnable in the batch setting with an ERM that is universal all-i-LOO stable. However that problem is not online learnable (by any deterministic algorithm) and there does not exist any (deterministic) algorithm that can be both uniform LOO stable and always AERM. When allowing randomized algorithms (convexifying the problem), the problem is online learnable via a uniform LOO stable RERM but there exists (randomized) universal all-i-LOO stable RERM that are not uniform-LOO stable that cannot achieve no regret.
\end{example}
\begin{proof}
This example was studied in both \citep{Kutin02,ShalevShwartz2009a}. Consider the hypothesis space $\mathcal{H} = \{0,1\}$, the instance space $\mathcal{Z} = \{0,1\}$ and the loss $f(h,z) = |h-z|$. As was shown in \citep{ShalevShwartz2009a} for the batch setting, an ERM for this problem is universally consistent and universally all-i-LOO stable, because removing a data point $z$ from the dataset can change the hypothesis only if there's an equal number of 0's and 1's (plus or minus one), which occurs with probability $O(\frac{1}{\sqrt{m}})$. \citet{ShalevShwartz2009a} also showed that the only uniform LOO stable algorithms on this problem must be constant (i.e. always return the same hypothesis $h$, regardless of the dataset), at least for large enough dataset, and hence cannot be an AERM.

It is also easy to see that this problem is not online learnable with any deterministic algorithm $\textbf{A}$. Consider an adversary who has knowledge of $\textbf{A}$ and picks the data points $z_i = 1-\textbf{A}(S_{i-1})$. Then algorithm $\textbf{A}$ incurs loss $\sum_{i=1}^m f(\textbf{A}(S_{i-1}),z_i) = m$, while there exists a hypothesis $h$ that achieves $\sum_{i=1}^m f(h,z_i) \leq \frac{m}{2}$. Hence for any deterministic algorithm $\textbf{A}$, there exists a sequence of data points such that $\frac{R_m(\textbf{A})}{m} \geq \frac{1}{2}$ for all $m$.

Now consider allowing randomized algorithms, in that we choose a distribution over $\{0,1\}$.  Allowing randomized algorithms makes the problem linear (and hence convex) in the distribution (by linearity of expectation) and makes the hypothesis space (the space of distributions on $\mathcal{H}$) convex. Let $p$ denote the probability of hypothesis 1. Then the problem can now be expressed with a hypothesis space $p \in [0,1]$ and the loss $f(p,z) = (1-p)z + p(1-z)$. 

This problem is obviously online learnable with a randomized uniform-LOO stable RERM (i.e. Hedge) that is uniform-LOO stable at rate $O(\frac{1}{\sqrt{m}})$ and no-regret at rate $O(\frac{1}{\sqrt{m}})$ using our previous results.

Even under this change, the previous ERM algorithm that is universally all-i-LOO stable would still choose the same hypothesis as before, i.e. $p$ would be always 0 or 1 and would not be uniform-LOO stable. That would also be the case even if we make it pick $p=\frac{1}{2}$ or some other intermediate value when there is an equal number of 0's and 1's. If we make it pick such intermediate value it would still be universal all-i-LOO stable as the hypothesis would still only change with small probability $O(\frac{1}{\sqrt{m}})$. However such algorithm cannot achieve no regret. Again if we pick the sequence $z_i = round(1-\textbf{A}(S_{i-1}))$, then whenever $i$ is even, the ERM use an odd number of data points and it must pick either 0 or 1 and would incur loss of 1. When $i$ is odd, there will be an equal number of 0's and 1's in the dataset (by the fact $\textbf{A}$ chooses the ERM at odd steps) and no matter what $p$ it picks it would incur loss of at least $\frac{1}{2}$. Thus $\frac{R_m(A)}{m} \geq \frac{1}{4}$ for all $m$.

We can also consider the following randomized RERM algorithm that uses only a convex regularizer:
$\textbf{A}(S) = \argmin_{p \in [0,1]} \sum_{i=1}^m f(p,z_i) + \sum_{i=0}^m \lambda_i |p - \frac{1}{2}|$. Let $\overline{z} = \frac{1}{m} \sum_{i=1}^m z_i$ and $\overline{\lambda} = \frac{1}{m} \sum_{i=0}^m \lambda_i$. Using the subgradient of this objective, we can easily show that $\textbf{A}(S)$ picks $\frac{1}{2}$ if $\overline{z} \in [\frac{1 - \overline{\lambda}}{2},\frac{1 + \overline{\lambda}}{2}]$, and otherwise picks either $p = 1$ if $\overline{z} > \frac{1 + \overline{\lambda}}{2}$ and $p = 0$ if $\overline{z} < \frac{1 - \overline{\lambda}}{2}$. This algorithm is not uniform-LOO stable, as for any regularizer $\lambda_m$ and large enough $m$ we can pick a dataset $S_m$ such that $S_{m-1}$ has $\overline{z} \in [\frac{1 - \overline{\lambda}}{2},\frac{1 + \overline{\lambda}}{2}]$ but $S_m$ has $\overline{z} \notin [\frac{1 - \overline{\lambda}}{2},\frac{1 + \overline{\lambda}}{2}]$, such that $f(\textbf{A}(S_m),z_m) = 0$ but $f(\textbf{A}(S_{m-1}),z_m) = \frac{1}{2}$. Hence $\epsilon_{\textrm{loo-stable}}(m) \geq \frac{1}{2}$. However it is universal all-i-LOO stable as the hypothesis would still only change with small probability $O(\frac{1}{\sqrt{m}})$ as in the previous case (we need to draw $m$ samples that has number of 1's $\frac{1 - \overline{\lambda}}{2}$ or $\frac{1 + \overline{\lambda}}{2}$, plus or minus one, for the hypothesis to change upon removal of a sample).

Furthermore this algorithm doesn't achieve no regret. Consider the sequence where whenever $\textbf{A}(S_{i-1})$ picks $\frac{1}{2}$ we pick $z_i = 1$ and whenever $\textbf{A}(S_{i-1})$ picks 1 we pick $z_i = 0$. It is easy to see that by the way this sequence is generated that the proportion of 1's $\overline{z}$ in $S_m$ will seek to track the boundary $\frac{1 + \overline{\lambda}}{2}$, where the algorithm switches between $p=\frac{1}{2}$ and $p=1$, as $m$ increases. Since $\overline{\lambda} \rightarrow 0$ as $m \rightarrow \infty$, then in the limit $\overline{z} \rightarrow \frac{1}{2}$. Since the sequence is such that everytime we generate a 0, the algorithm incurs loss of 1 and everytime we generate a 1 it incurs loss of $\frac{1}{2}$, then its average loss converges to $\frac{3}{4}$ but there's a hypothesis that achieves average loss of $\frac{1}{2}$ so the average regret converges to $\frac{1}{4}$. 
%
%
%
%
%
\end{proof}

This problem is insightful in a number of ways. First it shows that there are problems that are batch learnable that are not online learnable, but when considering randomized algorithms can become online learnable. Additionally it shows that a RERM that is universal all-i-LOO stable, the next weakest stability notion, cannot be sufficient to guarantee the algorithm achieves no regret. This shows we cannot guarantee no regret for any RERM using only universal all-i-LOO stability or any weaker notion of LOO stability. This also suggests that it might be necessary to have a notion of LOO stability that is at least stronger than all-i-LOO stability to guarantee no regret.

Another point reinforcing the fact that uniform-LOO stability might be necessary is that it is known that online learnability is not equivalent to batch learnability (as shown in the example below). Therefore, necessary stability conditions for online learnability should intuitively be stronger than for batch learnability.

\begin{example}
[Example taken from Adam Kalai and Sham Kakade] There exists a problem that is learnable in the batch setting but not learnable in the online setting by any deterministic or randomized online algorithm.
\end{example}
\begin{proof}
Consider a threshold learning problem on the interval $[0,1]$, where the true hypothesis $h^*$ is such that for some $x^* \in [0,1]$, $h(x) = 2 I(x \geq x^*) - 1$. Given an observation $z = (x,h^*(x))$, we define the loss incured by a hypothesis $h \in \mathcal{H}$ as $L(h,(x,h^*(x))) = \frac{1 - h(x)h^*(x)}{2}$, for $\mathcal{H} = \{ 2I(x \geq t)-1 | t \in [0,1] \}$ the set of all threshold functions on $[0,1]$. Since this is a binary classification problem and the VC dimension of threshold functions is finite (2), then we conclude this problem is batch learnable. In fact by existing results, it is batch learnable by an ERM that is all-i-LOO stable. However in the online setting consider an adversary who picks the sequence of inputs by doing the following binary search: $x_1 = \frac{1}{2}$ and $x_i = x_{i-1} - y_{i-1}2^{-i}$, and $y_{i} = -h_{i}(x_{i})$, so that the observation by the learner at iteration $i$ is $z_i = (x_i, y_i)$. This sequence is constructed so that the learner always incur loss of 1 at each iteration, and after any number of iterations $m$, the hypothesis $h = 2I(x \geq x_{m+1}) - 1$ achieves 0 loss on the entire sequence $z_1, z_2, \dots, z_m$. This implies the average regret of the algorithm is 1 for all $m$. Additionnally even if we allow randomized algorithms such that the prediction at iteration $i$ by the learner is effectively a distribution over $\{-1,1\}$ where $p_i$ denote the probability $P(h_i(x_i) = 1)$ for the distribution over hypotheses chosen by the learner, then the expected loss of the learner at iteration $i$ is $\frac{1 + y_i - 2p_iy_i}{2}$. If $x_i$ are chosen as before but $y_i = -I(p_i \geq 0.5)$. Then again at each iteration the learner must incur expected loss of at least $\frac{1}{2}$ but the hypothesis $h = 2I(x \geq x_{m+1}) - 1$ achieves loss of 0 on the entire sequence $z_1, z_2, \dots, z_m$. Hence the expected average regret is $\geq \frac{1}{2}$ for all $m$, so that with probability 1, the average regret of the randomized algorithm is $\geq \frac{1}{2}$ in the limit as $m$ goes to infinity. Hence we conclude that this problem is not online learnable.
\end{proof}
\subsection{Necessary Stability Conditions for Online Learnability in Particular Settings \label{secStabilityNecessary}}
\subsubsection{Binary Classification Setting}
We now show that if we restrict our attention to the binary classification setting ($f(h,(x,y)) = I(h(x) \neq y)$ for $y \in \{0,1\}$), online learnability is equivalent to the existence of a (potentially randomized) uniform LOO stable RERM.

Our argument uses the notion of Littlestone dimension, which was shown to characterize online learnability in the binary classification setting. \citet{BenDavid09} have shown that a classification problem is online learnable if and only if the class of hypothesis has finite Littlestone dimension.

By our current results we know that if there exists a uniform LOO stable RERM, the classification problem must be online learnable and thus have finite Littlestone dimension. We here show that finite Littlestone dimension implies the existence of a (potentially randomized) uniform LOO stable RERM. To establish this, we use the fact that when $\mathcal{H}$ is infinite but has finite Littlestone dimension ($\Ldim(\mathcal{H}) < \infty$), Weighted Majority can be adapted to be a no-regret algorithm by playing distributions over a fixed finite set of experts (of size $\leq m^{\Ldim(\mathcal{H})}$ when playing for $m$ rounds) derived from $\mathcal{H}$ \citep{BenDavid09}:
\begin{theorem}
For any binary classification problem with hypothesis space $\mathcal{H}$ that has finite Littlestone dimension $\Ldim(\mathcal{H})$ and number of rounds $m$, there exists a Randomized uniform-LOO stable RERM algorithm. In particular, it has no regret at rate $\epsilon_{\textrm{regret}}(t) \leq \sqrt{2 \log(m) \Ldim(\mathcal{H})} [\frac{3}{\sqrt{t}} + \frac{\log(t)}{2t} + \frac{1+2\ln(2)}{2t}]$ for all $t \leq m$.
\end{theorem}
\begin{proof}
The algorithm proceeds by constructing the same set of expert as in \citet{BenDavid09} from $\mathcal{H}$, which has number of experts $\leq m^{\Ldim(\mathcal{H})}$ for $m$ rounds. The previously mentioned Weighted Majority algorithm on this set achieves no regret at rate $\epsilon_{\textrm{regret}}(t) \leq \sqrt{2 \log(m) \Ldim(\mathcal{H})} [\frac{3}{\sqrt{t}} + \frac{\log(t)}{2t} + \frac{1+2\ln(2)}{2t}]$ for all $t \leq m$ (since the maximum instantaneous regret is 1) and is a Randomized uniform-LOO stable RERM as shown in theorem \ref{thmWM}.
\end{proof}

This result implies that finite littlestone dimension is equivalent to the existence of a (potentially randomized) uniform LOO stable RERM, and therefore that online learnability in the binary classification setting is equivalent to the existence of a (potentially randomized) uniform LOO stable RERM:
\begin{corollary}
A binary classification problem is online learnable if and only if there exists a (potentially randomized) uniform-LOO stable RERM.
\end{corollary}
\subsubsection{Problems with Sub-Exponential Covering}
For any $\epsilon > 0$, let $\mathcal{C}_\epsilon = \{ C \subseteq \mathcal{H} | \forall h' \in \mathcal{H}, \exists h \in C  s.t. \forall z \in \mathcal{Z}: |f(h,z) - f(h',z)| \leq \epsilon  \}$. $\mathcal{C_\epsilon}$ is the set of all subsets $C$ of $\mathcal{H}$ such that for any $h' \in \mathcal{H}$, we can find an $h \in C$ that has loss within $\epsilon$ of the loss of $h'$ at all $z \in \mathcal{Z}$.

We define the $\epsilon$-covering number of the tuple $(\mathcal{H}, \mathcal{Z}, f)$ as $N(\mathcal{H}, \mathcal{Z}, f, \epsilon) = \inf_{C \in \mathcal{C}_\epsilon} |C|$, i.e. the minimal number of hypotheses needed to cover the loss of any hypothesis in $\mathcal{H}$ within $\epsilon$. We will show that we can guarantee no-regret with a Randomized uniform-LOO stable RERM algorithm (e.g. Hedge) as long as there exists a sequence $\{\epsilon_i\}_{i=0}^\infty$ that is $o(1)$ and such that for any number of rounds $m$: $N(\mathcal{H}, \mathcal{Z}, f, \epsilon_m)$ is $o(\exp(m))$.
\begin{theorem}
Any learning problem (with instantaneous regret bounded by $B$) where there exists a sequence $\{\epsilon_m\}_{m=0}^\infty$ that is $o(1)$ and such that $\{ N(\mathcal{H}, \mathcal{Z}, f, \epsilon_m) \}_{m=0}^\infty$ is $o(\exp(m))$, is online learnable with a Randomized uniform-LOO stable RERM algorithm. In particular, when playing for $m$ rounds it has no regret at rate $\epsilon_{\textrm{regret}}(t) \leq B\sqrt{2 \log(N(\mathcal{H}, \mathcal{Z}, f, \epsilon_m))} [\frac{3}{\sqrt{t}} + \frac{\log(t)}{2t} + \frac{1+2\ln(2)}{2t}] + \epsilon_m$ for all $t \leq m$.
\end{theorem}
\begin{proof}
Suppose we know we must do online learning for $m$ rounds. Then we can construct an $\epsilon_m$-cover $C$ of $(\mathcal{H}, \mathcal{Z}, f)$ such that $C \subseteq \mathcal{H}$ and $|C| = N(\mathcal{H}, \mathcal{Z}, f, \epsilon_m)$. From the previous theorem, we know that running Hedge on the set $C$ guarantees that $\frac{1}{t}\sum_{i=1}^t \mathbb{E}_{h_i \sim P_{\theta_i}}[f(h_i,z_i)] - \inf_{h \in C} \frac{1}{t}\sum_{i=1}^t f(h,z_i) \leq B\sqrt{2 \log(N(\mathcal{H}, \mathcal{Z}, f, \epsilon_m))} [\frac{3}{\sqrt{t}} + \frac{\log(t)}{2t} + \frac{1+2\ln(2)}{2t}]$ for all $t \leq m$. By definition of $C$, $\inf_{h \in C} \frac{1}{t} \sum_{i=1}^t f(h,z_i) \leq \inf_{h \in \mathcal{H}} \frac{1}{t} \sum_{i=1}^t f(h,z_i) + \epsilon_m$ for all $t \leq m$. So we conclude $\epsilon_{\textrm{regret}}(t) \leq B\sqrt{2 \log(N(\mathcal{H}, \mathcal{Z}, f, \epsilon_m))} [\frac{3}{\sqrt{t}} + \frac{\log(t)}{2t} + \frac{1+2\ln(2)}{2t}] + \epsilon_m$ for all $t \leq m$.
\end{proof}

This theorem applies to a large number of settings. For instance, if we have a problem where $f(\cdot,z)$ is $K$-Lipschitz continuous at all $z \in \mathcal{Z}$ with respect to some norm $||\cdot||$ on $\mathcal{H}$, and $\mathcal{H} \subseteq \mathbb{R}^d$ for some finite $d$ and has bounded diameter $D$ under $||\cdot||$ (i.e. $\sup_{h,h' \in \mathcal{H}} ||h-h'|| \leq D$). Then $N(\mathcal{H}, \mathcal{Z}, f, \epsilon)$ is $O(K (\frac{D}{\epsilon})^d)$ for all $\epsilon \geq 0$. Choosing $\epsilon_m = \frac{1}{m}$ implies we can achieve no regret at rate $\epsilon_{\textrm{regret}}(t) \leq O(B \sqrt{\frac{\log(K) + d\log(mD)}{t}})$ for all $t \leq m$. This notion also allows to handle highly discontinuous loss functions. For instance consider the case where $\mathcal{Z} = \mathcal{H} = \mathbb{R}$ and the loss $f(h,z) = 1 - I( h \in \mathbb{Q} ) I(z \in \mathbb{Q}) - I( h \notin \mathbb{Q} ) I(z \notin \mathbb{Q})$, i.e. the loss is 0 if both $h$ and $z$ are rational, or both irrational, and the loss is 1 is one is rational and the other irrational. In this case, the set $C = \{1,\sqrt{2}\}$ is an $\epsilon$-cover of $\{\mathcal{H}, \mathcal{Z}, f\}$ for any $\epsilon > 0$ and thus we can achieve no-regret at rate $O(\frac{1}{\sqrt{m}})$ by running Hedge on the set $C$.

\section{Conclusions and Open Questions}
In this paper we have shown that popular online algorithms such as FTL, FTRL, Mirror Descent, gradient-based methods and randomized algorithms like Weighted Majority and Hedge can all be analyzed purely in terms of stability properties of the underlying batch learning algorithm that picks the sequence of hypotheses (or distribution over hypotheses). In particular, we have introduced the notion of online stability, which is sufficient to guarantee online learnability in the general learning setting for the class of RERM and RSLM algorithm. Our results allow to relate a number of learnability results derived for the batch setting to the online setting. There are a number of interesting open questions related to our work. First, it is still an open question to know whether for the general class of always AERM (at $o(1)$ rate) it is sufficient to be online stable (at $o(1)$ rate) to guarantee no regret, or show a counter-example that proves otherwise. The presented examples seem to suggest that a problem is online learnable only if there exists a uniform-LOO stable or online stable (and always AERM) algorithm, or at least with some form of LOO stability in between online stable and all-i-LOO stable. This has been verified in the binary classification setting where we have shown that online learnability is equivalent to the existence of a potentially randomized uniform-LOO stable RERM. While we haven't been able to provide necessary conditions for online learnability in the general learning setting, we have shown that all problems with a sub-exponential covering are all online learnable with a potentially randomized uniform-LOO stable RERM. An interesting open question is whether the notion of sub-exponential covering we have introduced turns out to be equivalent to online learnability in the general learning setting. If this is the case, this would establish immediately that existence of a (potentially randomized) uniform-LOO stable RERM is both sufficient and necessary for online learnability in the general learning setting. 

\bibliography{biblio.bib}
\end{document}